\documentclass[11pt]{article}

\usepackage{fullpage}
\usepackage{cite,bm}
\usepackage{amsmath, amsfonts, amssymb, amsthm}
\usepackage{graphicx}
\usepackage{hyperref}
\usepackage{multirow}
\hypersetup{
    colorlinks,%
    citecolor=black,%
    filecolor=black,%
    linkcolor=black,%
    urlcolor=black
}
\usepackage{color}
\usepackage{framed}

\newtheorem{thm}{Theorem}[section]

\newtheorem{rem}[thm]{Remark}

\numberwithin{equation}{section}

\def\XXint#1#2#3{{\setbox0=\hbox{$#1{#2#3}{\int}$}
  \vcenter{\hbox{$#2#3$}}\kern-.5\wd0}}

\newcommand{\va}{\varepsilon}                      
\newcommand{\be}{\begin{equation}}      \newcommand{\ee}{\end{equation}}

\newcommand{\R}{\mathbb{R}}

\begin{document}

\date{\today}
\title{\textbf{Enhanced Expressive Power and Fast Training of Neural Networks by Random Projections}}

\author{\medskip Jian-Feng Cai\footnote{Supported in part by Hong Kong RGC grants GRF 16300616 and GRF 16306317.},\quad Dong Li\footnote{Supported in part by Hong Kong RGC grant GRF 16307317.}, \quad Jiaze Sun \quad and \quad Ke Wang\footnote{Supported in part by HKUST Initiation Grant IGN16SC05.}}
\maketitle

\begin{abstract}
Random projections are able to perform dimension reduction efficiently for datasets with nonlinear low-dimensional structures. One well-known example is that random matrices embed sparse vectors into a low-dimensional subspace nearly isometrically, known as the restricted isometric property in compressed sensing. In this paper, we explore some applications of random projections in deep neural networks. We provide the expressive power of fully connected neural networks when the input data are sparse vectors or form a low-dimensional smooth manifold. We prove that the number of neurons required for approximating a Lipschitz function with a prescribed precision depends on the sparsity or the dimension of the manifold and weakly on the dimension of the input vector. The key in our proof is that random projections embed stably the set of sparse vectors or a low-dimensional smooth manifold into a low-dimensional subspace. Based on this fact, we also propose some new neural network models, where at each layer the input is first projected onto a low-dimensional subspace by a random projection and then the standard linear connection and non-linear activation are applied. In this way, the number of parameters in neural networks is significantly reduced, and therefore the training of neural networks can be accelerated without too much performance loss. 
%We discovered that in fully connected neural networks, the computation speed using compressive sensing is much faster.
\end{abstract}

\section{Introduction}
Over the past few years, learning via multiple-layer neural network has been widely studied and has achieved unprecedented success. It has many important applications in image recognition, speech recognition, and natural language processing. 

One of the fundamental theoretical question in deep learning is the expressive power of a neural network, which describes its ability to approximate functions. The celebrated universal approximation theorem, which was proved by Cybenko \cite{Cybenko}, Hornick \cite{Hornik} et al, Funahashi \cite{Funahashi} and Barron \cite{Barron},  states that sufficiently large shallow (that is, depth-2 or equivalently, one hidden layer) neural networks can approximate any continuous function on a bounded domain to arbitrary accuracy. But the required size of such networks can be exponentially increasing with respect to the dimension. Indeed, Eldan and Shamir \cite{ES} proved that there is a continuous function expressed by a small depth-3 feedforward neural networks which cannot be approximated by any shallow network to more than a certain constant accuracy, unless its width grows exponentially in the dimension. This shows the power of depth for feedforward neural network. Lu-Pu-Wang-Hu-Wang \cite{LPWHW} studied the expressive power of neural networks from the width point of view. They shown that there exists a class of width-$O(k^2)$ shallow ReLU network that cannot be approximated by any width-$O(k^{1.5})$ and depth-$k$ neural network.

However, the data input in the real world applications are usually structured. For example, images modelled as piecewise smooth functions can have sparse representations under certain orthonormal bases or frames \cite{mallat2008wavelet}. This means that the intrinsic dimension of the input data is significantly smaller than the ambient space dimension. This fact is often ignored in aforementioned classical approximation results. The expressive power of a neural network may be improved by exploring the structure of the input data. In this direction, Shaham-Cloninger-Coifman \cite{shaham2016provable} studied approximations of functions on a smooth $k$-dimensional submanifold embedded in $\R^d$. They constructed a depth-4 network and controlled the error of its approximation, where the size of their network depends on $k$ but just weakly on $d$. Chui-Lin-Zhou \cite{ChuiLinZhou} studied the expressive power of neural networks in the regression setting when the samples are located approximately on some unknown manifold. They showed that the error of the approximation of their trained depth-3 neural network to the regression function depends on the number of samples, and the dimension of the manifold instead of the ambient dimension.

In this paper, we consider a different approach to analyze the theoretical performance of neural networks with structured input data. Based on our analysis, we propose a new architecture of neural networks, for which the training can be significantly accelerated compared to conventional fully connected or convolutional neural networks. Our main idea is to use linear random projections developed in compressed sensing \cite{foucart2013mathematical}. 

For simplicity, we assume that the input data are sparse vectors, namely, $k$-sparse vectors in $\R^d$. Using the theory of compressed sensing, one can construct a random projection onto an $O(k\log(d/k))$-dimensional space that satisfies the so-called restricted isometric property (RIP), saying that the random projection is nearly isometric when restricted to sparse vectors. Therefore, to get an efficient function approximation, we can first compress the sparse high-dimensional input vectors to low-dimensional ones without changing the metric too much, and then apply a standard neural network with low-dimensional vectors as inputs. In this way, the expressive power is the same as that of the neural network with $O(k\log(d/k))$-dimensional input. In other words, we obtain neural networks with expressive power depending on $k$ and weakly on $d$. Contrary to the work in \cite{shaham2016provable,ChuiLinZhou}, the neural networks we constructed can be very deep with a fixed width. Our approach works not only for sparse inputs but also for a large class of structures of input data, for example, when the input data are sampled from a low dimensional manifold.

The idea of using random projection can also be incorporated into fully connected and convolutional neural networks  to get new architectures for multi-layer neural networks. At each layer, we first apply a random projection to capture the intrinsic data structure, and then the standard linear transform and nonlinear activation follow. This will improve the overall computational efficiency of the neural networks, since the number of parameters are significantly reduced. In particular, compared to standard fully connected and convolutional neural networks, the training is accelerated drastically. We will also demonstrate the new neural network can achieve comparable accuracy to the original ones.

The rest of this paper is organized as follows. In Section \ref{sec:cs}, we give a brief introduction to random projections and their stable embedding. %We will see that Gaussian random matrices is nearly isometric when restricted to the set of sparse vectors, known as RIP in compressed sensing. We also see that random projections can stably embed a smooth manifold into a low-dimensional subspace. 
In Section \ref{sec:theoimp}, we present our theoretical results on the expressive power of neural networks for structured input data. In Section \ref{sec:compimp}, we incorporate random projections into each layer of full connected and convolutional neural networks to get better architecture of deep neural networks. Section \ref{sec:exp} is devoted to numerical experiments demonstrating the efficiency and effectiveness of the resulting neural networks.

\section{Random Projections and the Restrictive Isometry Property}\label{sec:cs}

We will use random projections to study the expressive power and accelerate the training of neural networks. Random projections are able to embed datasets with a non-linear low-dimensional structure into a low-dimensional subspace while almost keeping the metric. In this section, we give a review on related results. 

\subsection{Compressed Sensing and RIP}

Compressed sensing \cite{foucart2013mathematical,candes2006robust} is a signal processing technique that enables acquiring compressible signals from a much smaller number of linear samples than the ambient dimension of signals. It has numerous applications in imaging \cite{qu2015accelerated,lustig2007sparse,candes2006robust}. Compressed sensing takes advantage of the fact that most signals of interest in practice are compressible: there are only a few nonzero or big elements when the signals are represented over a certain dictionary such as wavelet basis. The key concept in compressed sensing is the restricted isometric property (RIP), under which many algorithms are able to reconstruct the underlying signal stably. One of the most powerful results in compressed sensing is that some family of random matrices with very few rows will satisfy RIP with high probability. 

A signal in $\bm{x}\in\R^d$ is compressible if $\|\bm{W}\bm{x}\|_0\leq k$ for some $k\ll d$ for some linear transform $\bm{W}\in\R^{m\times d}$. Here the $\ell_0$-norm $\|\cdot\|_0$ stands for the number of nonzeros of a vector. We also call such an $\bm{x}$ a $k$-sparse signal. In other words, a signal is compressible if it is sparse under certain linear transform. This assumption holds true for a wide variety of classes of signals. For example, piecewise smooth signals are (nearly) sparse under the representation of orthogonal wavelets or wavelet frames such as the curvelet \cite{mallat2008wavelet}. Actually, the sparsity assumption is the foundation of many models and approaches in modern signal processing and imaging. 

Since the degree of freedom in a $k$-sparse signal is only $k$ with $k\ll d$, it is possible to acquire the signal efficiently by $n$ linear samples with $n\ll d$. This is exploited by compressed sensing. In the encoding stage of compressed sensing, we acquire a $k$-sparse sparse $\bm{x}$ by $\bm{b}:=\bm{Ax}$, where $\bm{A}\in\R^{n\times d}$ is a sampling matrix with each row corresponding to one linear sample. Compared to traditional signal processing where the full sample of $\bm{x}$ is required, compressed sensing can save sampling costs significantly and can be applied to a wider range of imaging scenarios for which the full sampling is prohibited. In the decoding stage, one wants to recover the $k$-sparse signal $\bm{x}$ from $\bm{b}$. Various approaches are available with recovery performance guarantee, including convex optimization based approaches \cite{candes2006robust} and non-convex ones \cite{needell2009cosamp,blumensath2009iterative,foucart2009sparsest}; see also the book \cite{foucart2013mathematical} and the references therein. 

A key concept in compressed sensing theory is the restricted isometric property (RIP) introduced by E. Cand\`es and T. Tao \cite{CT}. For simplicity, we assume the sparse transform $\bm{W}=\bm{I}$, i.e., the underlying signal $\bm{x}$ satisfies
$$
\bm{x}\in \mathcal{S}_k:=\{\bm{y}\in\R^d~:~\|\bm{y}\|_0\leq k\}.
$$
Then, to have a successful recovery of any signal $\bm{x}\in\mathcal{S}_k$ from $\bm{b}=\bm{A}\bm{x}$, we should at least require that the sampling operator $\bm{A}$ is injective on $\mathcal{S}_k$. That is, it is necessary $\|\bm{A}(\bm{x}_1-\bm{x}_2)\|_2>0$ for any $\bm{x}_1,\bm{x}_2\in\mathcal{S}_k$ satisfying $\bm{x}_1\neq\bm{x}_2$. However, there will always be noise in the practical measurements, and also the injectivity may be too restrictive to design recovery algorithms. To have a stable recovery with practical algorithms, we need to relax the restricted injectivity to the following restricted isometric property (RIP):
\begin{equation}\label{eq:RIP}
(1-\delta_s)\|\bm{y}\|_2\le \|\bm{A}\bm{y}\|_2\le(1+\delta_s)\|\bm{y}\|_2\quad\mbox{for all }\bm{y}\in\mathcal{S}_s,
\end{equation}
where $\delta_s\in(0,1)$. Here we assume $\bm{A}$ is normalized so that its restricted eigenvalues are concentrated around $1$. If we choose $s=2k$, then RIP \eqref{eq:RIP} implies that, for any $k$-sparse vectors $\bm{x}_1$ and $\bm{x}_2$, $\|\bm{A}(\bm{x}_1-\bm{x}_2)\|_2$ is almost the same as $\|\bm{x}_1-\bm{x}_2\|_2$. In other words, if $\bm{A}$ satisfies RIP, then the application of $\bm{A}$ preserves the metric of the set $\mathcal{S}_k$ inherited from $\R^d$. This enables us to design efficient and stable algorithms for the recovery of sparse signals with a theoretical performance guarantee \cite{candes2006robust,needell2009cosamp,blumensath2009iterative,foucart2009sparsest}. For example, it was shown that if $\bm{A}$ satisfies RIP with $\delta_{2s}<\sqrt2-1$, then the solution of the following $\ell_1$-norm minimization gives a stable recovery of the $k$-sparse vector $\bm{x}$
\begin{equation}\label{eq:l1min}
\min_{\bm{y}\in\R^d}\|\bm{y}\|_1\quad\mbox{s.t.}\quad\|\bm{A}\bm{y}-\bm{b}\|_2\leq\sigma,
\end{equation}
where $\bm{b}=\bm{A}\bm{x}+\bm{\epsilon}$ are the noisy linear samples of $\bm{x}$ with the noise $\bm{\epsilon}$ satisfying $\|\bm{\epsilon}\|_2\leq\sigma$. Actually, even when $\bm{x}$ is not exactly in $\mathcal{S}_k$ but only close to it, \eqref{eq:l1min} is still able to give a faithful recovery of $\bm{x}$.

\subsection{Random Projections Satisfy RIP}

An important question is then to find matrices $\bm{A}\in\R^{n\times d}$ with good RIP constants using the smallest possible $n$. Since computing the RIP constants is strongly NP-hard, it is very difficult to use numerical methods to construct RIP matrices. Also, any existing deterministic matrices satisfying RIP will not have an optimal $m$. The best known deterministic RIP matrices have a number of rows $n\geq O(s^2)$, and it is still an open problem to construct a deterministic matrix satisfying RIP.  A major breakthrough in compressive sensing is the use of random matrices to construct RIP matrices with optimal $m$. In particular, let $\bm{A}\in\R^{n\times d}$ be a random matrix whose entries are independent Gaussian random variables with mean $0$ and variance $1/n$. Then, with overwhelming probability, $\bm{A}$ satisfies RIP with constant $\delta_{s}>0$ provided $n\geq O(\delta_s^{-2}s\log(d/s))$. This result is summarized in the following theorem.
\begin{thm}[Theorem 9.2 in \cite{foucart2013mathematical}]\label{thm:GaussianRIP}
Let $\bm{A}\in\R^{n\times d}$ be a random matrix whose entries are independent Gaussians with mean-$0$ variance-$1/n$. Then, there exists a universal constant $C>0$ such that $\bm{A}$ satisfies RIP \eqref{eq:RIP} with constant $\delta_s\in(0,\delta)$ with probability at least $1-2e^{-\frac{\delta^2}{2C}n}$ provided
$$
n\geq 2C\delta^{-2}s\log(ed/s).
$$
\end{thm}
The theorem was first proved in \cite{CT} with a weaker bound in a more restrictive setting, and simple proofs can be found in \cite{foucart2013mathematical,baraniuk2008simple}. Besides random Gaussian matrices, there exist other types of random matrices satisfying RIP with $n\sim O(s\log^{\alpha}(d))$ for some $\alpha>0$, such as subGaussian, Bernoulli, and random rows of discrete Fourier transform matrices \cite{foucart2013mathematical,baraniuk2008simple,rudelson2008sparse}.

The RIP above can be extended to general cases where the sparsifying transform $\bm{W}$ is not necessarily the identity. When $\bm{W}\in\R^{d\times d}$ is orthogonal, RIP \eqref{eq:RIP} can be adapted to
\begin{equation}\label{eq:RIPW}
(1-\delta_s)\|\bm{y}\|_2\le \|\bm{A}\bm{y}\|_2\le(1+\delta_s)\|\bm{y}\|_2\quad\mbox{for all }\bm{y}\in\mathcal{S}^{\bm{W}}_s,
\end{equation}
for some $\delta_s\in(0,1)$, where 
$$
\mathcal{S}^{\bm{W}}_s=\{\bm{y}\in\R^d~:~\|\bm{W}\bm{y}\|_0\leq k\}.
$$
Since Gaussian random variables are unitary invariant, a simple calculation and Theorem \ref{thm:GaussianRIP} give that, for any orthogonal $\bm{W}$, a Gaussian random matrix $\bm{A}\in\R^{n\times d}$ satisfies the generalized RIP \eqref{eq:RIPW} with high probability provided $n\sim O(s\log(d/s))$. When $\bm{W}\in\R^{m\times d}$ forms a tight frame (i.e., $\bm{W}^T\bm{W}=\bm{I}$), the generalized RIP was studied in \cite{candes2011compressed,rauhut2008compressed}. Gaussian random matrices satisfy the generalized RIP there with high probability and optimal $n$.

\subsection{RIP on Smooth Manifolds}

Besides sparse signal models, there is another common model called manifold signal model. This model generalizes the notion of concise signal structure beyond the framework of bases and representations. It arises in broad cases, for example, where a $k$-dimensional parameter can be identified that carries the relevant information about a signal that changes as a continuous function of these parameters. In general, this dependence may not be neatly reflected in a sparse set of transform coefficients. In \cite{BW}, Baraniuk and Wakin proposed a approach for nonadaptive dimensionality reduction of manifold-modeled data, where they demonstrated that a small number of random linear projections can preserve key information about a manifold-modeled signal. To state their results, we need a few definitions for a Riemannian manifold.

Let $\mathcal{M}$ be a $k$-dimensional compact Riemannian submanifold embedded  in $\R^d$. The condition number is defined as $1/\tau$, where $\tau$ is the largest number having the following property: for every $r<\tau$, the tubular neighborhood of $\mathcal{M}$ of radius $r$ in $\R^d$ defined as $\{\bm{x}+\eta\in\R^d: \bm{x}\in \mathcal{M}, \eta\in \mathrm{Tan}_{\bm{x}}^\perp,\|\eta\|_2<r \}$, where $\mathrm{Tan}_{\bm{x}}^\perp$ denotes the set of vectors normal to the tangent space at $\bm{x}$, is embedded in $\R^d$. Given $T>0$, the geodesic covering number $G(T)$ of $\mathcal{M}$ is defined as the smallest number such that there exists a finite set $\mathcal{A}\subset\mathcal{M}$ of $G(T)$ points so that $$\min_{\bm{a}\in \mathcal{A}}d_{\mathcal{M}}(\bm{a},\bm{x})\le T$$ for all $\bm{x}\in \mathcal{M}$, where $d_{\mathcal{M}}(\bm{a},\bm{x})$ is the geodesic distance between $\bm{a}$ and $\bm{x}$. We say that $\mathcal{M}$ has geodesic covering regularity $R$ for resolutions $T\le T_0$ if $$ G(T)\le \frac{R^kVk^{k/2}}{T^k},$$
where $V$ is the volume of $\mathcal{M}$. 

\begin{thm}[Theorem 3.1 in \cite{BW}]\label{thm:manifoldrp}
Let $\mathcal{M}$ be a compact $k$-dimensional Riemannian submanifold of $\R^d$ having condition number $1/\tau$, volume $V$, and geodesic covering regularity $R$. Fix $0<\delta<1$ and $0<\rho<1$. Let $\bm{A}=\sqrt{\frac{d}{n}}\bm{\Phi}$, where $\bm{\Phi}\in\mathbb{R}^{n\times d}$ is a random orthoprojector with
\[
n=O\left(\frac{k\log(dVR\tau^{-1}\delta^{-1})\log(1/\rho)}{\delta^2}\right).
\]
If $n\le d$, then with probability at least $1-\rho$, the following statement holds: For every $x,y\in\mathcal{M}$,
\begin{equation}\label{eq:manifoldsrip}
(1-\delta)\|\bm{x}-\bm{y}\|_2\le \|\bm{A}\bm{x}-\bm{A}\bm{y}\|_2\le (1+\delta)\|\bm{x}-\bm{y}\|_2.
\end{equation}
\end{thm}
Hence, if $k\ll d$, then we reduce the dimension of input data from $d$ to $O(k\log d)$.

The diameter of the manifold is defined by
\[
\operatorname{diam}(M)=\sup_{x,y\in M} d_M(x,y),
\]
where $d_M(x,y)$ is the geodesic distance between $x,y$ on $M$.

\section{Improved Expressive Power of Neural Networks by Random Projections}\label{sec:theoimp}

In this section, we use the random projections discussed in the previous section to explore the expressive power of neural networks for functions on datasets with a low-dimensional structure in the ambient space $\R^d$. 

\subsection{Fully-Connected Neural Networks (FCNN)} 
There are many artificial neural network architectures available, such as the fully-connected neural network (FCNN), the convolutional neural network (CNN), and the recurrent neural network. We will study the expressive power a multi-layer fully-connected neural network with the rectified linear unit (ReLU) as the activation function.

The ReLU is so far the most popular activation function for deep neural networks, and it is the positive part of its argument. More precisely, let $\bm{z}\in\R^{\ell}$, and the ReLU is defined by
\[
\mbox{ReLU}(\bm{z}):=\left[\begin{matrix}\max\{0,z_1\}\cr\vdots\cr\max\{0,z_{\ell}\}\end{matrix}\right],\quad\forall\, \bm{z}\in\R^{\ell}.
\]
%Since any affine transform is a linear transform in a space with dimension larger than $1$, without loss of generality we consider a multi-layer FCNN without bias term. 
Let $\bm{x}\in\R^d$ be the input. Then it generates outputs $\bm{x}^{(l)}$ at the $l$-th layer recursively by: $\bm{x}^{(0)}=\bm{x}$, and
$$
\bm{x}^{(l)}=\mbox{ReLU}\left(\bm{W}^{(l)}\bm{x}^{(l-1)}\right),\quad l=1,2,\ldots,L,
$$
where $\bm{W}^{(l)}:\R^{d_{l-1}}\to\R^{d_l}$ is the affine transformation at layer $l$. The final output $y$ of the neural network is
$$
y=\bm{W}^{(L+1)}\bm{x}^{(L)},
$$
where $\bm{W}^{(L+1)}:\R^{d_{L}}\to\R$ is an affine transformation. 
Therefore, the function represented by the fully-connected neural network is
\[
\bm{W}^{(L+1)}\circ\mbox{ReLU}\circ \bm{W}^{(L)} \circ\cdots\circ\mbox{ReLU}\circ \bm{W}^{(1)}.
\]
The number $L$ is usually called the depth of the network, and the $d_l$'s are called the widths. The sum $\sum_{l=1}^{L}d_l$ is called the number of neurons of the network. If $L=1$, then the network is called a shallow neural network. If $L>1$, it is called a deep neural network.

\subsection{Expressive Power of Neural Networks} 

As mentioned in the introduction, by the celebrated universal approximation theorem,  we know that sufficiently large shallow neural networks can approximate any continuous function on a bounded domain to arbitrary accuracy.  That is, for a continuous function $f\in C([-1,1]^d)$ and for every $\va>0$, there exists a shallow ReLU neural network $f_\va$ such that 
\[
\|f-f_\va\|_{L^\infty([-1,1]^d)}:=\max_{\bm{x}\in [-1,1]^d}|f(\bm{x})-f_\va(\bm{x})|<\va.
\] 
However, the universal approximation theorem does not tell the number of neurons that $f_\va$ has, or equivalently the approximation accuracy. There have been many literatures on studying the number of neurons, or the approximation accuracy, of (either shallow or deep) neural networks since the work \cite{Barron} by Barron, and we know now that the number of neurons that the  neural network $f_\va$ needs will depend on the regularity (e.g., the modulus continuity) of the function $f$. 

Barron \cite{Barron} first gave a quantitative approximation rate in $L^2$ norm for shallow neural networks, assuming the function $f$ has bounded first moment of the magnitude of the Fourier transform. If $f$ is $r$ times differentiable, then Mhaskar \cite{Mhaskar} obtained an optimal quantitative approximation rate in $L^2$ norm, that is, for every $\va>0$, there exists a shallow neural network $f_\va$ with $O(\va^{-d/r})$ neurons such that
\[
\|f-f_\va\|_{L^2([-1,1]^d)}<\va.
\]
If $f$ is $C^2$ and has bounded Hessian, then Shaham-Cloninger-Coifman \cite{shaham2016provable} proved the same quantitative approximation rate in $L^\infty$ norm. Their result also holds if $f$ is supported in a lower dimensional manifold. In a recent work \cite{Hanin} of Hanin, several approximation results of ReLU neural networks were obtained for continuous functions, convex functions, and smooth functions, respectively. In particular, if $f$ is bounded and Lipschitz continuous, then Theorem 1 in \cite{Hanin} tells that for every $\va>0$, there exists a ReLU neural network with $Cdd! \va^{-d}$ neurons, where $C$ is a positive constant depending only on the Lipschitz constant of $f$ and is independent of $d$, such that
\[
\|f-f_\va\|_{L^\infty([-1,1]^d)}<\va.
\]
Recall that a function $f:E\to\R$ is called Lipschitz continuous on the set $E\subset\R^d$ if 
\[
\operatorname{Lip}(f):=\sup\left\{\frac{|f(\bm{x})-f(\bm{y})|}{\|\bm{x}-\bm{y}\|_2}\ |\ \bm{x}, \bm{y}\in E,\  \bm{x}\neq \bm{y}\right\}<\infty,
\]
where $\|\bm{x}-\bm{y}\|_2$ is the distance between $\bm{x}$ and $\bm{y}$ in $\R^d$.
If $f$ is Lipschitz continuous, then $\operatorname{Lip}(f)$ defined above is called the Lipschitz constant. Note that Theorem 1 in \cite{Hanin} is stated for positive functions, but it is clearly true for bounded function as well by subtracting a large constant from the neural network.

\subsection{Improved Expressive Power for Functions with Sparse Inputs}
The aforementioned results on expressive power assumes that the domain of functions is $[-1,1]^d$. However, in real applications, the input data $\bm{x}$ is often structured. For example, in an image recognition task, the input data are digital images with $d$ pixels, which obviously are not arbitrary in $[-1,1]^d$. By considering the structure of the input data, we expect to obtain better results on the expressive power of neural networks than results reviewed in the previous section, which ignore structures of input data.

There are several approximation results taking into account the structure of the input. Shaham-Cloninger-Coifman \cite{shaham2016provable} assumed that the domain of the function is a  $k$-dimensional smooth submanifold embedded in $\R^d$ and proved that the size of their constructed approximating networks depends on $k$ but just weakly on $d$. In \cite{ChuiLinZhou}, Chui-Lin-Zhou used neural networks to approximate regression functions, where the training samples are located approximately on some
unknown manifold. It was shown there that the error of the approximation of their trained depth-3 neural network to the regression function depends on the number of samples and the dimension of the manifold instead of the ambient dimension. In both of these two papers, multi-layer neural networks are constructed, for which each hidden layer is endowed with a specific learning task.

Here we provide an improved expressive power by assuming the sparsity of the input vectors and using random projections. Our assumption is motivated by the facts that many learning tasks are with images or audio signals as inputs and that images and audio signals have sparse representation under suitable basis. For simplicity, we assume the input vectors are in $\mathcal{S}_k=\{\bm{y}\in\R^d:~\|\bm{y}\|_0\leq k\}$ with $k\ll d$, and the extension to the sparse case under a linear transformation (i.e., the input vectors are in $\mathcal{S}^{\bm{W}}_s$) is straightforward.

We will show in the below that, for Lipschitz continuous functions that are defined on $\mathcal{S}_k$, we can choose a neural network with $Cnn! \va^{-n}$ neurons, where $n=O(k\log(d/k))$ and $C$ is a positive constant depending only on the Lipschitz constant of $f$, to approximate $f$ with accuracy $\va$. Our proof will make use of a theorem of McShane \cite{McShane} and Whitney \cite{Whitney} on the extension of Lipschitz functions, which states that any Lipschitz continuous function defined on an arbitrary subset of $\R^n$ can be extended to be a Lipschitz continuous function in $\R^n$ with the same Lipschitz constant. (See also Theorem 3.1 in the book \cite{EG} of Evans and Gariepy.)

\begin{thm}\label{thm:appx}
Let $\mathcal{S}_k=\{\bm{y}\in\R^d:~\|\bm{y}\|_0\leq k\}$ with $k\ll d$. Suppose $f: \mathcal{S}_k\to\R$ is a Lipschitz continuous function with Lipschitz constant $\operatorname{Lip}(f)$, that is,
\[
\operatorname{Lip}(f)=\sup_{\bm{x},\bm{y}\in\mathcal{S}_k,\ \bm{x}\neq\bm{y}}\frac{|f(\bm{x})-f(\bm{y})|}{\|\bm{x}-\bm{y}\|_2}<\infty.
\]
%Let $\delta>0$, and $\bm{A}\in\R^{n\times d}$ satisfies the restricted isometric property \eqref{eq:RIP} with constant $\delta$.   
Then for sufficiently large $m$,
$$
\inf_{f_0\in F_m^d}\sup_{\bm{x}\in\mathcal{S}_k\cap [-1,1]^d}|f(\bm{x})-f_0(\bm{x})|
\leq C \operatorname{Lip}(f) k^{\frac 32}\log(d/k)m^{-\frac{C}{k\log(d/k)}},
$$
where $F_m^d$ is the set of functions represented by ReLU fully-connected neural networks with $m$ neurons and $d$ inputs, and $C$ is a universal positive constant.
\end{thm}
\begin{rem}
In applications, the sparsity is usually much smaller than the dimension, i.e., $k\ll d$. Then in order to have
\[
\inf_{f_0\in F_m^d}\sup_{\bm{x}\in\mathcal{S}_k\cap [-1,1]^d}|f(\bm{x})-f_0(\bm{x})|\le \va,
\]
it suffices to require that $\log m = 2k(\log k+\log\log d-\frac{1}{2}\log \va)\cdot \log d$, i.e., 
\[
m= d^{2k(\log k+\log\log d)-k\log \va},
\]
which is significantly smaller than exponential functions of $d$.
\end{rem}
\begin{proof}[Proof of Theorem \ref{thm:appx}]
Fix a $\delta\in(0,1)$, say $\delta=1/2$. Theorem \ref{thm:GaussianRIP} implies that there exists a matrix $\bm{A}\in\R^{n\times d}$, where $n=C_0k\log(d/k)$ with a universal positive constant $C_0$, satisfying the restricted isometric property \eqref{eq:RIP} with constant $\delta=1/2$.
Since $\bm{A}$ satisfies RIP,  the map $ \bm{x}\mapsto \bm{A}\bm{x}$ is a bijection from $\mathcal{S}_k$ to $\bm{\Omega}:=\{\bm{A}\bm{x}\ |\ \bm{x}\in\mathcal{S}_k\}$. For every $\bm{y}\in \bm{\Omega}$, define
\[
g(\bm{y})=f(\bm{x}),
\]
where $\bm{x}$ is the unique element in $\mathcal{S}_k$ such that $\bm{A}\bm{x}=\bm{y}$. 

Since $f$ is a Lipschitz continuous on $\mathcal{S}_k$ with Lipschitz constant $\operatorname{Lip}(f)$ and $\bm{A}$ satisfies RIP, it is elementary to show that $g$ is Lipschitz continuous on $\bm{\Omega}$ with Lipschitz constant at most $2\operatorname{Lip}(f)$. Notice that $\bm{\Omega}\subseteq \R^n$. Then by the theorem of McShane \cite{McShane} and Whitney \cite{Whitney}, we can extend $g$ to be a Lipschitz continuous function in $\R^n$ with the same Lipschitz constant. 

For every $\bm{x}\in\mathcal{S}_k\cap [-1,1]^d$, we have 
$$
\|\bm{A}\bm{x}\|_\infty\le \|\bm{A}\bm{x}\|_2\le \frac 32 \|\bm{x}\|_2\le \frac{3\sqrt{k}}{2}.
$$ 
Hence, 
$$
\{\bm{A}\bm{x}\ |\ \bm{x}\in\mathcal{S}_k\cap [-1,1]^d\}\subset \left[-\frac{3\sqrt{k}}{2},\frac{3\sqrt{k}}{2}\right]^n.
$$ 
Now we consider $$\tilde g(\bm{y})=g\left(\frac{3\sqrt{k}}{2}(2\bm{y}-\bm{1})\right),$$ where $\bm{1}=(1,\cdots,1)^T\in\R^n$. The function $\tilde g$ is Lipschitz continuous on $\R^n$ with Lipschitz constant $\operatorname{Lip}(\tilde g)=3\sqrt{k}\operatorname{Lip}(g)\le 6\sqrt{k}\operatorname{Lip}(f)$.  By Theorem 1 in \cite{Hanin}, we have that for every $\va>0$, there exists a ReLU neural network $g_\va$ with at most $m=C_1 n n! (6\sqrt{k}\operatorname{Lip}(f))^n\va^{-n}$ neurons such that
\[
\|\tilde g-g_\va\|_{L^\infty([0,1]^n)}\le \va,
\]
where $C_1>0$ is a universal constant.  In terms of $f$, we obtain
\[
\inf_{h\in F_m^n}\sup_{\bm{x}\in\mathcal{S}_k\cap [-1,1]^d}|f(\bm{x})-h(\bm{A}\bm{x})|
\leq C_2\sqrt{k}\operatorname{Lip}(f) nm^{-\frac{1}{n}},
\]
where $C_2>0$ is an another universal constant. Let $f_0=h(\bm{A}\cdot)$. Since $h\in F_m^n$, we can rewrite it as $h=\bm{W}^{(L+1)}\circ\mbox{ReLU}\circ \bm{W}^{(L)} \circ\cdots\circ\mbox{ReLU}\circ \bm{W}^{(1)}$. Therefore, $f_0=\bm{W}^{(L+1)}\circ\mbox{ReLU}\circ \bm{W}^{(L)} \circ\cdots\circ\mbox{ReLU}\circ (\bm{W}^{(1)}\circ\bm{A})\in F_{m}^d$, which concludes the proof. 
\end{proof}

%\begin{rem}
%A key ingredient in the proof is that Gaussian random projection satisfies RIP for $\bm{x}\in\mathcal{S}_k$. In fact, Gaussian random projection satisfies RIP on some other subset of $\R^d$. One well-known such a subset is the set $\mathcal{M}_{r}^{d_1\times d_2}$ of all $d_1\times d_2$ (here $d=d_1d_2$) matrices with a fixed rank $r$, which forms a smooth manifold of dimension $s:=r(d_1+d_2-r)$ in $\R^d$. As long as $n\geq O(s)$, a Gaussian random matrix $\bm{A}\in\mathbb{R}^{n\times d}$ will satisfies RIP on $\mathcal{M}_{r}^{d_1\times d_2}$ with high probability. A similar proof will imply that approximating a Lipschitz function $f$ on $\mathcal{M}_{r}^{d_1\times d_2}$ up to an $\varepsilon$-precision needs only $m=s^{O(-\log\varepsilon)}$ neurons. 
%\end{rem}

\subsection{Improved Expressive Power for Functions on Smooth Manifolds}

Using similar techniques, we can also improve the expressive power of FCNN with inputs from a compact $k$-dimensional Riemannian manifold embedded in $\R^d$. We assume that the underlying function is Lipschitz continuous, which is weaker than that in \cite{shaham2016provable}. It turns out that the number of neurons required for an $\varepsilon$-approximation in the infinity-norm depends weakly on $d$. Our main tool in the proof is Theorem \ref{thm:manifoldrp}, which states that random projections are stable embeddings.

\begin{thm}\label{thm:appx-manifold}
Let $(\mathcal{M}, g)$ be a compact $k$-dimensional Riemannian submanifold of $\R^d$. Let $f: \mathcal{M}\to\R$ be a Lipschitz continuous function with Lipschitz constant $\operatorname{Lip}(f)$, that is, 
%$\sup_{\bm{y}\in\mathcal{M}} |f(\bm{y})|<\infty$ and
\[
\operatorname{Lip}(f)=\sup_{\bm{x},\bm{y}\in\mathcal{M},\ \bm{x}\neq\bm{y}}\frac{|f(\bm{x})-f(\bm{y})|}{\|\bm{x}-\bm{y}\|_2}<\infty.
\]
%Let $\delta>0$, and $\bm{A}\in\R^{n\times d}$ satisfies the restricted isometric property \eqref{eq:RIP} with constant $\delta$.   
Then for sufficiently large $m$,
$$
\inf_{f_0\in F_m^d}\sup_{\bm{x}\in \mathcal{M}}|f(\bm{x})-f_0(\bm{x})|
\leq C\operatorname{Lip}(f) \operatorname{diam}(M) \log(2dVR\tau^{-1})m^{-\frac{C}{\log(2dVR\tau^{-1})}},
$$
where $F_m^d$ is the set of functions represented by ReLU fully-connected neural networks with $m$ neurons and $d$ inputs, and $\tau$ is the condition number of $\mathcal{M}$ in $\R^d$, $V$ is the volume of $(\mathcal{M},g)$, $R$ is the geodesic covering regularity of $\mathcal{M}$, $\operatorname{diam}(M)$ is the diameter of $(\mathcal M,g)$ and $C$ is a universal positive constant.
\end{thm}
\begin{proof}
Fix a $\delta\in(0,1)$, say $\delta=1/2$. Theorem \ref{thm:manifoldrp} implies that there exists a matrix $\bm{A}\in\R^{n\times d}$  satisfies the restricted isometric property \eqref{eq:manifoldsrip} with constant $\delta=\frac 12$. Here, $n=C_0k\log(2dVR\tau^{-1})$ with $C_0>0$ is a universal constant. Since $\bm{A}$ satisfies RIP,  the map $ \bm{x}\mapsto \bm{A}\bm{x}$ is a bijection from $\mathcal{M}$ to $\bm{\Omega}:=\{\bm{A}\bm{x}\ |\ \bm{x}\in \mathcal{M}\}\subset\R^n$. For every $\bm{y}\in \bm{\Omega}$, define
\[
\phi(\bm{y})=f(\bm{x}),
\]
where $\bm{x}$ is the unique element in $\mathcal{M}$ such that $\bm{A}\bm{x}=\bm{y}$. As in the proof of Theorem \ref{thm:appx}, $\phi$ is Lipschitz continuous on $\bm{\Omega}$ with Lipschitz constant at most $2\operatorname{Lip}(f)$ and can be extended to be a Lipschitz continuous function in $\R^n$ with the same Lipschitz constant. 

For every $\bm{x}_1, \bm{x}_2\in \mathcal{M}$,
\[
\|\bm{A}\bm{x}_1-\bm{A}\bm{x}_2\|_\infty\le \|\bm{A}\bm{x}_1-\bm{A}\bm{x}_2\|_2\le \frac 32 \|\bm{x}_1-\bm{x}_2\|_2\le \frac 32\operatorname{diam}(M).
\]
Now we consider $$\tilde \phi(\bm{y})=\phi\left(\frac 32\operatorname{diam}(M)\bm{y}+\bm{y}_0\right),$$ where $\bm{y}_0\in\R^n$ is chosen such that $\left\{\frac 32\operatorname{diam}(M)\bm{y}+\bm{y}_0,  \bm{y}\in [0,1]^n\right\}\subset\bm{\Omega}$. Then $\tilde \phi$ is Lipschitz continuous on $\R^n$ with Lipschitz constant $\operatorname{Lip}(\tilde \phi)=\frac 32\operatorname{diam}(M)\operatorname{Lip}(\phi)=3\operatorname{diam}(M)\operatorname{Lip}(f)$. By Theorem 1 in \cite{Hanin}, we have that for every $\va>0$, there exists a ReLU neural network $\phi_\va$ with at most $m=C_1 n n! (3 \operatorname{diam}(M)\operatorname{Lip}(f))^{n}\va^{-n}$ neurons such that
\[
\|\tilde \phi-\phi_\va\|_{L^\infty([0,1]^n)}\le \va,
\]
where $C_1>0$ is a universal constant.  In terms of $f$, we obtain
\[
\inf_{h\in F_m^n}\sup_{\bm{x}\in\mathcal{M}}|f(\bm{x})-h(\bm{A}\bm{x})|
\leq C \operatorname{Lip}(f) \operatorname{diam}(M) nm^{-\frac{1}{n}},
\]
where $C_2>0$ is an another universal constant. Now we can conclude the proof in the same way as in the proof of Theorem \ref{thm:appx}. 
\end{proof}

\begin{rem}
The volume of $\mathcal{M}$ and the geodesic covering regularity of $\mathcal{M}$ do not depend on the ambient space $\R^d$, and hence, do not depend on $d$.  In the definition of the condition number of $\mathcal{M}$ in $\R^d$, it appears to depend on $\R^d$. However, the dependence on $d$ is very weak. For example, by the Nash embedding theorem \cite{Nash}, $\mathcal{M}$ can be isometrically embedded in $\R^{k(3k+11)/2}$. If $d\ge \frac{k(3k+11)}{2}$, then the condition number of $\mathcal{M}$ in $\R^d$ is the same as the condition number of $\mathcal{M}$ in $\R^{k(3k+11)/2}$.
\end{rem}

\begin{rem}
%The Lipschitz continuity of $f$ in Theorem \ref{thm:appx-manifold} is with respect to the Euclidean distance $\|\cdot\|_2$. 
There are examples shown in \cite{NK,SVK} that one pixel change will make deep neural networks misclassify natural images. Such changes
 induce a severe jump in the Lipschitz constant (which is defined with respect to the Euclidean distance as in Theorem \ref{thm:appx-manifold}). From the practical point of view, the natural distance of the  input data should perhaps be the Euclidean distance rather than the geodesic distance on the manifold. That is part of the reason why the Lipschitz continuity of $f$ in Theorem \ref{thm:appx-manifold} is stated in terms of the Euclidean distance $\|\cdot\|_2$.
\end{rem}

\section{Accelerate the Training of Neural Networks by Random Projections}\label{sec:compimp}
In this section, we use the random projections to accelerate the training of neural networks. We present in detail our implementation of random projections in both fully-connected and convolutional neural networks. We shall also provide estimates to demonstrate that our scheme indeed achieves significant reduction in computational complexity and number of parameters.

\subsection{Fully-connected Neural Networks (FCNN)}
The main idea in the proof of Theorems \ref{thm:appx} and \ref{thm:appx-manifold} is to use random projections to reduce the number of neurons. In practice, we will also use random projections to  reduce the number of neurons and hence the number of parameters of FCNN. Consequently, the training of FCNN is significantly accelerated. 

Here for simplicity, we assume that the input vectors are sparse under a suitable linear transformation, which is a common assumption for digital images and signals. It is similar when input signals are on a smooth manifold. The bulk of the computations of an FCNN comes from the matrix multiplication in each layer. Recall that a multi-layer FCNN produces outputs by
\begin{equation}\label{eq:fcnn}
	\bm{x}^{(l)} = \mathrm{ReLU}\left(\bm{W}^{(l)}\bm{x}^{(l-1)}\right), \quad\mbox{for}~l=1,\ldots,L,
\end{equation}
where $\bm{W}^{(l)}:\mathbb{R}^{d_{l-1}}\to\R^{d_{l-1}}$ is an affine transformation, with $d_{l-1}$ and $d_{l}$ respectively being the dimensions of the input $\bm{x}^{(l-1)}$ and output $\bm{x}^{(l)}$. The final output is $y=\bm{W}^{(L+1)}\bm{x}^{(L)}$ with $\bm{W}^{(L+1)}:\R^{d_L}\to\R$ an affine transformation. 

We modify \eqref{eq:fcnn} by random projection as in the following. At layer-$1$, the input vector $\bm{x}^{(0)}=\bm{x}\in\mathcal{S}_k$ is sparse. According to Theorem \ref{thm:GaussianRIP}, if we choose a suitable $n_0$, then with high probability a Gaussian random matrix $\bm{A}^{(1)}\in\mathbb{R}^{n_0\times d_0}$ embeds $\mathcal{S}_k$ to $\R^{n_0}$ nearly isometrically. Therefore, $\bm{A}^{(1)}$ reduce the dimension of $\mathcal{S}_k$ without too much information loss. We then apply an affine transformation $\bm{U}^{(1)}:\R^{n_0}\to\R^{d_1}$ on the embedded subspace $\R^{n_0}$. In other words, we replace $\bm{W}^{(1)}$ by $\bm{U}^{(1)}\circ\bm{A}^{(1)}$. In this way, there are only $d_1(n_0+1)$ parameters in layer-$1$, which is significantly smaller than $d_1(d_0+1)$ in \eqref{eq:fcnn}. Since ReLU set negative entries to $0$, the outputs of each layer are sparse vectors. The same parameter reduction procedure as in layer-$1$ is applied to each layer. In particular, at layer-$l$, we replace $\bm{W}^{(l)}$ by $\bm{U}^{(l)}\circ\bm{A}^{(l)}$, where $\bm{U}^{(l)} : \R^{n_{l-1}}\to\R^{d_l}$ is an affine transformation to be trained and $\bm{A}^{(l)}\in\mathbb{R}^{n_{l-1}\times d_{l-1}}$ is a given matrix drawn from random Gaussian distribution. The number of parameters at layer-$l$ is reduced from $d_l(d_{l-1}+1)$ in \eqref{eq:fcnn} to $d_l(n_{l-1}+1)$. Altogether, we propose the following FCNN
\begin{equation}\label{eq:fcnn2}
	\bm{x}^{(l)} = \mathrm{ReLU}\left(\bm{U}^{(l)}\circ\bm{A}^{(l)}\bm{x}^{(l-1)}\right), \quad\mbox{for}~l=1,\ldots,L,
\end{equation}
where $\bm{U}^{(l)}:\mathbb{R}^{n_{l-1}}\to\R^{d_{l}}$ is an affine transformation to be trained from the data, and $\bm{A}^{(l)}\in\R^{n_{l-1}\times d_{l-1}}$ is a fixed matrix with entries drawn from i.i.d. random Gaussian distribution. The final output is $y=\bm{W}^{(L+1)}\bm{x}^{(L)}$ with $\bm{W}^{(L+1)}:\R^{L}\to\R$ an affine transformation to be trained. 

Under this scheme, the number of parameters can be significantly reduced so long as $n_l$'s are kept small. In fact, it can be immediately calculated that the number of parameters of the original FCNN \eqref{eq:fcnn} is $(d_{L}+1)+\sum_{l=1}^{L}(d_{l-1}+1)d_{l}$, whereas that of the modified network \eqref{eq:fcnn2} is $(d_{L}+1)+\sum_{l=1}^{L}(n_{l-1}+1)d_{l}$. In practice, since it is usually the case that $d_1$ and $d_{l-1}$ are large and the input $\bm{x}^{(l)}$ is sparse, $n_{l-1}\ll d_{l-1}$ can be easily satisfied. Therefore, the reduction of number of parameters by \eqref{eq:fcnn2} is significant.
%Moreover, the floating point operations (FLOPs) for the execution of the original FCNN \eqref{} is $2d_ld_{l-1}$ (considering only the principal terms), whereas that of the RP scheme is $2k(d_l+d_{l-1})$, which is $(1/d_{l-1}+1/d_l)k$ of the original number. The calculation for the number of parameters is identical, and the RP scheme uses roughly $k/d_{l-1}$ of the original amount. In practice, both these ratios can be kept much smaller than 1, since it is usually the case that $d_l$ and $d_{l-1}$ are large and the input $\bm{x}^{(l)}$ is sparse, thus $k\ll \min\{d_{l-1},d_l\}$ can be easily satisfied.

\subsection{Convolutional Neural Networks (CNN)}
Since convolution is linear in nature, it is possible to adapt our random projection scheme from FCNNs to CNNs. Again for simplicity, consider the $l$-th layer of a CNN with a ReLU activation:
\begin{equation}\label{eq:cnn}
\mathcal{X}_{\cdot\cdot j}^{(l)} = \mathrm{ReLU}\left(\mathcal{F}_{\cdot\cdot\cdot j}^{(l)} * \mathcal{X}^{(l-1)}\right),\quad 1\leq j \leq c_l.
\end{equation}
Here for $i=l-1$ or $l$, $\mathcal{X}^{(i)}\in\mathbb{R}^{m\times m\times c_i}$ is the output tensor at the $i$-th layer, with height and width $m$, and $c_i$ number of channels; $\mathcal{F}^{(l)}\in\mathbb{R}^{h\times h\times c_{l-1}\times c_{l}}$ is a trainable tensor consisting of $c_l$ filters of height and width $h$ and depth $c_{l-1}$. In addition, the ``$\cdot$'' notation in the subscripts means all entries in that axis are included, for instance $\mathcal{X}_{\cdot\cdot j}^{(l)}$ is simply the $m\times m$ matrix at the $j$-th channel of the tensor $\mathcal{X}^{(l)}$, and $\mathcal{F}_{\cdot\cdot\cdot j}^{(l)}$ is the $j$-th filter at the $l$-th layer. We shall describe two possible modifications of the CNN by random projections.

\subsubsection{Approach I: Direct Extension}
Upon realizing that convolution is essentially a matrix multiplication acting on different patches of the input tensor, \eqref{eq:cnn} can be rewritten more succinctly as
\begin{equation}\label{eq:matcnn}
\bm{X}^{(l)} = \mathrm{ReLU}\left(\bm{F}^{(l)}\tilde{\bm{X}}^{(l-1)}\right),
\end{equation}
where $\tilde{\bm{X}}^{(l-1)}\in\mathbb{R}^{c_{l-1}h^2\times m^2}$ is a matrix whose columns are vectorized $h\times h\times c_{l-1}$ patches of $\mathcal{X}^{(l-1)}$ to be convolved with the filters; $\bm{F}^{(l)}\in\mathbb{R}^{c_l\times c_{l-1}h^2}$ is a matrix wherein each row is given by $\bm{F}_{j\cdot}^{(l)}=\text{vec}({\mathcal{F}_{\cdot\cdot\cdot j}^{(l)}})^T$; $\bm{X}^{(l)}\in\mathbb{R}^{c_l \times m^2}$ is the matrix whose rows are vectorized channels of $\mathcal{X}^{(l)}$, namely $\bm{X}^{(l)}_{j\cdot} = \text{vec}(\mathcal{X}^{(l)}_{\cdot\cdot j})^T$. As argued in the previous section, columns of $\tilde{\bm{X}}^{(l-1)}$ are sparse. Therefore, following \eqref{eq:fcnn2}, instead of $\bm{F}^{(l)}$, we first do dimension reduction of the sparse vectors by a random Gaussian matrix $\bm{A}^{(l)}$, followed by a linear transformation $\bm{U}^{(l)}$ in the reduced subspace. We obtain
\begin{equation}\label{eq:rpcnn}
\bm{X}^{(l)} = \mathrm{ReLU}\left(\bm{U}^{(l)}\bm{A}^{(l)}\tilde{\bm{X}}^{(l-1)}\right),
\end{equation}
where, similar to the FCNN scheme, $\bm{U}^{(l)}\in\mathbb{R}^{c_l\times n_{l-1}}$ is trainable, and $\bm{A}^{(l)}\in\mathbb{R}^{n_{l-1}\times c_{l-1}h^2}$ is a fixed matrix whose entries are drawn from i.i.d. random Gaussian distribution. The effect of the larger variance of $\bm{A}^{(l)}$ can be nullified by batch normalization. Similar to the FCNN case, the number of parameters has been reduced because of the approximation $\bm{F}^{(l)}$ by $\bm{U}^{(l)}\bm{A}^{(l)}$ with a small $n_{l-1}$.

\subsubsection{Approach II: Per-channel Extension}
In the first approach, columns in $\bm{A}^{(l)}$ that correspond to one channel of the input tensor are independent from those corresponding to another. In other words, different channels of the input tensor are essentially assigned different ``$\bm{A}$''s. This means that the scheme does not take into account the correlations amongst the input channels, however in practice, for instance, the RGB channels of an input image are closely related. As a result, this might diminish the expressive power of the network. To remedy this, we can consider apply random projections separately for each channel, and then perform a summation over all the per-channel outputs. Let $\hat{\bm{X}}^{(l-1)}\in\mathbb{R}^{h^2\times m^2\times c_{l-1}}$ be the tensor such that the columns of $\hat{\bm{X}}_{\cdot\cdot j}^{(l-1)}$ are vectorized $h\times h$ patches from the $j$-th channel of $\mathcal{X}^{(l-1)}$. Then our modified scheme can be written as
\begin{equation}\label{eq:rpcnn2}
\bm{X}^{(l)} = \mathrm{ReLU}\left(\sum_{j=1}^{c_{l-1}} \hat{\bm{U}}_{\cdot\cdot j}^{(l)}\hat{\bm{A}}^{(l)}\hat{\bm{X}}_{\cdot\cdot j}^{(l-1)}\right),
\end{equation}
where $\hat{\bm{U}}^{(l)}\in\mathbb{R}^{c_{l}\times n_{l-1} \times c_{l-1}}$ is trainable, and $\hat{\bm{A}}^{(l)}\in\mathbb{R}^{n_{l-1}\times h^2}$ is a given matrix whose entries are randomly drawn from i.i.d. Gaussian distribution.

A small price for this approach is an increase in computational and model complexity compared to the first approach \eqref{eq:rpcnn}, but it still achieves substantial, albeit less dramatic, reduction over the original CNN \eqref{eq:cnn} as long
as $n_{l-1}$ is small.
%It can be estimated that this approach uses approximately $(1/h^2 + 1/c_l)k$ of the original FLOPs, and $k/h^2$ of the original number of parameters. Comparing these two ratios with those of the first approach, the only difference is the absence of the $c_{l-1}$ term in some of the denominators. As a result, $k$ needs to be kept even smaller than that in the first approach, namely $k\ll\min\{c_l, h^2\}$. However, it can be demonstrated in the experiments that this in no way causes the scheme to be less effective.

\section{Experiments}\label{sec:exp}
In this section, we demonstrate that our schemes in Section \ref{sec:compimp} indeed achieve significant reduction in model and computational complexity while causing minimal loss in classification accuracy. 

\begin{table}[h]
	\centering
	\begin{tabular}{c|l|l}
		\hline
		Data & \multicolumn{1}{c|}{FCNN} & \multicolumn{1}{c}{CNN}\\ 
		\hline
		\multirow{4}{*}{MNIST} & INPUT $\to$ FC 1024, ReLU & INPUT $\to$ 5x5 CONV 64, BN, ReLU, 3x3 MP\\
		& \phantom{INPUT} $\to$ FC 1024, ReLU & \phantom{INPUT} $\to$ 5x5 CONV 128, BN, ReLU, 3x3 MP\\
		& \phantom{INPUT} $\to$ FC 10 & \phantom{INPUT} $\to$ FC 512, BN, ReLU \\
		& & \phantom{INPUT} $\to$ FC10 \\
		\hline
		\multirow{5}{*}{CIFAR-10} & INPUT $\to$ FC 4096, ReLU & INPUT$\to$ 5x5 CONV 128, BN, ReLU, 3x3 MP \\
		& \phantom{INPUT} $\to$ FC4096 ReLU & \phantom{INPUT} $\to$ 5x5 CONV 192, BN, ReLU, 3x3 MP \\
		& \phantom{INPUT} $\to$ FC10 & \phantom{INPUT} $\to$ 5x5 CONV 256, BN, ReLU, 3x3 MP\\
		& & \phantom{INPUT} $\to$ FC 512, BN, ReLU\\
		& & \phantom{INPUT} $\to$ FC 10\\
		\hline
	\end{tabular}
	\caption{The specifications of the models used in the experiment. ``INPUT'' refers to the input layer (no operation is performed here). ``FC $d$'' means fully-connected layer with $d$ hidden units. ``$h$x$h$ CONV $c$'' is a convolutional layer with filter size $h$ and $c$ output channels. ``$h$x$h$ MP'' is a max-pooling layer with filter size $h$. ``BN'' is a batch normalization layer. ``ReLU'' is the rectified linear unit.}
	\label{table:specs}
\end{table}

We used two well-known data sets, MNIST and CIFAR-10. MNIST is a collection of $28 \times 28$ images of handwritten digits from 0 to 9, in which there are 60,000 training examples 10,000 testing ones. CIFAR-10 consists of $32\times 32\times 3$ images, where 3 indicates the three RGB channels, and there are 50,000 training examples and 10,000 test ones in this data set. The experiments were conducted using the TensorFlow framework. The specifications of the models are detailed in Table \ref{table:specs}. We used stochastic gradient descent with 0.9 momentum for training FCNNs, and Adam optimization algorithm for CNNs. A suitable initial learning rate was chosen for each experiment and halved every 2,400 steps. All models were trained for 20 epochs and the results are summarized in Table \ref{table:results}. We see that with a sufficiently large $n_l=n$ our modified models use significantly small parameters and computational cost while achieving similar results. 

\begin{table}[!h]
	\centering
	\begin{tabular}{c|c|c|c|c|c|c}
		\hline
		Architecture & Data & \multicolumn{2}{c|}{Configuration} & Top-1 Error (\%) & FLOPs ($10^6$) & Parameters ($10^3$) \\ \hline
		\multirow{10}{*}{FCNN} & \multirow{5}{*}{MNIST} %(initial learning rate $=$ 0.05)}
		& \multicolumn{2}{c|}{Original} & 1.44 & 3.70 & 1,851.39 \\ \cline{3-7}
		& &\multicolumn{2}{c|}{$n=250$} & 1.48 & 1.93 & 512.00 \\
		& &\multicolumn{2}{c|}{$n=150$} & 1.83 & 1.16 & 307.20 \\
		& &\multicolumn{2}{c|}{$n=100$} & 2.29 & 0.77 & 204.80 \\
		& &\multicolumn{2}{c|}{$n=50$} & 3.42 & 0.39 & 102.40 \\ \cline{2-7}
		& \multirow{5}{*}{CIFAR-10} %(initial learning rate $=$ 0.01)}
		& \multicolumn{2}{c|}{Original} & 40.69 & 58.72 & 29,360.13 \\ \cline{3-7}
		& &\multicolumn{2}{c|}{$n=1500$} & 41.34 & 46.08 & 12,288.00 \\
		& &\multicolumn{2}{c|}{$n=1000$} & 41.80 & 30.72 & 8,192.00 \\
		& &\multicolumn{2}{c|}{$n=700$} & 42.89 & 21.50 & 5,734.40 \\
		& &\multicolumn{2}{c|}{$n=500$} & 43.05 & 15.36 & 4,096.00 \\ 
		\hline
		\multirow{20}{*}{CNN} & \multirow{9}{*}{MNIST}
		& \multicolumn{2}{c|}{Original} & 0.48 & 82.79 & 206.40 \\ \cline{3-7}
		& & \multirow{4}{*}{Approach I} & $n=15$ & 0.56 & 12.25 & 2.88 \\
		& & & $n=10$ & 0.66 & 8.17 & 1.92 \\
		& & & $n=5$ & 0.84 & 4.08 & 0.96 \\
		& & & $n=3$ & 1.05 & 2.45 & 0.58 \\ \cline{3-7}
		& & \multirow{4}{*}{Approach II} & $k=10$ & 0.50 & 39.78 & 82.56 \\
		& & & $n=7$ & 0.60 & 27.85 & 57.79 \\
		& & & $n=5$ & 0.59 & 19.89 & 41.28 \\
		& & & $n=3$ & 0.87 & 11.93 & 24.77 \\ \cline{2-7}
		& \multirow{9}{*}{CIFAR-10}
		& \multicolumn{2}{c|}{Original} & 15.15 & 491.52 & 1,852.80 \\ \cline{3-7}
		& & \multirow{4}{*}{Approach I} & $n=40$ & 21.41 & 111.98 & 23.04\\
		& & & $n=25$ & 22.20 & 69.99 & 14.40 \\
		& & & $n=15$ & 24.69 & 41.99 & 8.64 \\
		& & & $n=10$ & 27.13 & 28.00 & 5.76 \\ \cline{3-7}
		& & \multirow{4}{*}{Approach II} & $n=15$ & 15.62 & 331.01 & 1,111.68\\
		& & & $n=10$ & 17.67 & 220.67 & 741.12 \\
		& & & $n=7$ & 18.92 & 154.47 & 518.78 \\
		& & & $n=5$ & 20.49 & 110.34 & 370.56 \\ 
		\hline
	\end{tabular}
	\caption{The model errors and complexities under various configurations. ``Original'' means no RP scheme is applied. ``FLOPs'' counts the number of multiplications and additions in the model, and ``Parameters'' is the number of parameters. Note that ``FLOPs'' and ``Parameters'' do not take into account biases, BN, the last layer in FCNNs (which computes the predicted probabilities), or the FC layers in CNNs. We choose $n_l=n$ for all $n$.}
	\label{table:results}
\end{table}\par

\bibliographystyle{abbrv}
\bibliography{NNreference}

%\index{Bibliography@\emph{Bibliography}}%

\bigskip

\bigskip

\noindent Jian-Feng Cai,~ Dong Li, ~Jiaze Sun, ~Ke Wang

\noindent Department of Mathematics, The Hong Kong University of Science and Technology\\
Clear Water Bay, Kowloon, Hong Kong\\

\noindent Email: \textsf{jfcai@ust.hk}, \textsf{madli@ust.hk}, \textsf{jsunau@connect.ust.hk}, \textsf{kewang@ust.hk}

\end{document}